\documentclass{article}
\usepackage{macros}

\title{List Learning with Attribute Noise}
\author{Mahdi Cheraghchi\thanks{University of Michigan, Ann Arbor, {\tt mahdich@umich.edu}}
\and Elena Grigorescu\thanks{Purdue University, {\tt elena-g@purdue.edu}}
\and Brendan Juba\thanks{Washington University in St.\ Louis, {\tt bjuba@wustl.edu}}
\and Karl Wimmer\thanks{Duquesne University, {\tt wimmerk@duq.edu}}
\and Ning Xie\thanks{Florida International University, {\tt nxie@cis.fiu.edu}}
}

\date{}

\begin{document}

\maketitle
\begin{abstract}
   We introduce and study the model of list learning with attribute noise. Learning with attribute noise was introduced by Shackelford and Volper (COLT 1988) as a variant of PAC learning, in which the algorithm has access to noisy examples and  uncorrupted labels, and the goal is to recover an accurate hypothesis. Sloan (COLT 1988) and Goldman and Sloan (Algorithmica 1995) discovered information-theoretic limits to learning in this model, which have impeded further progress. In this article we 
   extend the model to that of list learning, drawing inspiration from the list-decoding model in coding theory, and its recent variant studied in the context of learning.
   On the positive side, we show that sparse
   conjunctions can be efficiently list learned
   under some assumptions on the underlying
   ground-truth distribution. 
   On the negative side, 
   our results show that even in the list-learning model, efficient learning of parities and majorities is not possible regardless of the representation used.

\end{abstract}

\section{Introduction}

\newcommand{\calC}{\mathcal C}
\newcommand{\calR}{\mathcal R}

We study the attribute-noise PAC learning model, introduced by Shackelford and Volper~\cite{ShackelfordV88}, in which learning must be achieved despite the presence of errors that corrupt the \emph{attributes} of the data (instead of the \emph{labels} of the data that are more commonly used in the learning with error setting). The inherent difficulty in learning with attribute noise has been formalized by Sloan~\cite{Sloan88} and Goldman and Sloan~\cite{GoldmanS95} by showing information-theoretic barriers: in the presence of attribute noise, regardless of how much data is used, it is impossible to identify which representations are accurate. Historically, similar issues of identifiability were tackled in coding theory by relaxing the notion of a solution to that of {\em list decoding}~\cite{elias,wozencraft}; more recently, a similar notion of \emph{list-learning} has been proposed to provide solutions in other learning settings where a correct solution simply cannot be identified from the given data~\cite{BalcanBV08, CharikarSV17, DiakonikolasKS18, karmalkar2019list, raghavendra2019list}. We further discuss this previous work in Section~\ref{sec:relwork}.
In this work, we ask when and to what extent it is possible to overcome the non-identifiability barrier posed by attribute noise by relaxing the solutions to lists of representations of Boolean functions.

In the attribute-noise model the task is to learn a labeling function given labeled examples, where the examples may have corrupted entries.
More formally, the algorithm has access to pairs $(\tilde{x}, c(x))$, where $x=(x_1, x_2, \ldots, x_n)\in X$ is chosen uniformly and independently from an unknown distribution $\calD$ over  $X$, $c\in \calC$ is an unknown labeling function from a concept class $\calC$ over domain $X$, and  $\tilde{x}$ is obtained from $x$ by applying a noise vector $\rho=(\rho_1, \rho_2, \ldots, \rho_n)$ from a noise distribution that affects the coordinates (a.k.a. attributes) of $x$; the goal is to output,  with probability $1-\delta$,  a hypothesis $c'$ that is $(1-\epsilon)$-accurate with respect to $c$ over $\calD$, namely $\Pr_{x\in \calD}[c(x)=c'(x)]>1-\epsilon$.  Hence, while in the standard PAC-learning model of Valiant~\cite{Valiant84} the algorithm  has access to $\tilde{x}=x$ --- namely actual samples from the input distribution, in the attribute-noise version, the algorithm only has access to a noisy version of $x$, making the task of learning the labeling function significantly more difficult.
 
 The attribute-noise model captures a setting in which one seeks an accurate model of dependencies in the ``ground truth'' process captured by $\calD$ and $c$, in spite of errors in the {\em recording} of the data. For example, this formulation is appropriate for the task of formulating models in data-driven science; a small {\em list} of candidate functions in such a setting then corresponds to a list of possible hypotheses for further investigation. It stands in contrast to the (much easier) {\em label noise} model, which captures the task of making accurate predictions from the observed data while the observed data is generated from an unknown concept which may not match $c$. Indeed, if one is only interested in forecasting or building a device that works directly with the noisy data $\tilde{x}$ produced by given real-world sensors, such a setting may be captured by a suitable label-noise model. We stress that since accuracy in the attribute-noise model is assessed with respect to $\calD$, which is never observed directly, the attribute-noise model is {\em not} captured by the label noise model, and is indeed much more challenging than the label noise model.

All previous work studies concept classes over Boolean attributes $x_i\in\{0,1\}$ for all $i\in [n]$, and Boolean labeling functions $c:\{0,1\}^n\rightarrow \{0, 1\}$.  Specifically, Shackelford and Volper~\cite{ShackelfordV88} show that under {\em uniform random attribute noise}, where the noise flips each coordinate independently with probability $p\in [0,1]$, it is possible to learn $k$-DNF expressions and conjunctions efficiently, if the noise rate $p$ is known by the algorithm. In fact, the knowledge of $p$ is not necessary for efficient learning, as proved by Goldman and Sloan~\cite{GoldmanS95}. They further consider {\em product random attribute noise} on conjunctions, where coordinates are affected independently by noise of possibly different rates $p_i$, and prove that if these rates are unknown, and if  $p_i>2\epsilon$ in each coordinate, then it is information-theoretically impossible to recover any $(1-\epsilon)$-accurate hypothesis. Hence, regardless of the running time of the algorithm, and the number of samples received, the algorithm is unable to output a good answer. On the other hand, if the noise rates are known, Decatur and Gennaro~\cite{DecaturG95} provide efficient algorithms for PAC-learning conjunctions and $k$-DNF formulas. Further, \cite{BshoutyJT03} studies noise distributions that are unconstrained or unknown, but where the examples come from the uniform distribution. 

We emphasize that the attribute-noise model is not captured by noisy-PAC. Indeed, the celebrated results of  Anguin and Laird \cite{AngluinL87} show that learning is the noisy PAC model is information theoretically possible for any noise rate $\rho<1/2$, and in fact $k$-CNF and $k$-DNFs can be learned efficiently in this high-noise regime. Again, this is in contrast with the attribute-noise setting where identifiability is not possible for unknown noise rate $\rho > 2\eps$ per coordinate \cite{GoldmanS95}. One can also view attribute noise as an intermediate between noisy PAC and malicious noise, in which the assumption is that $1-\rho$ fraction of the output is correct, and the remaining $\rho$ fraction may be completely irrelevant. Kearns and Li \cite{KearnsL93} show that in this model in  order to identify an $\eps$-accurate hypothesis one must have $\rho<\eps/(1+\eps).$

Motivated by its applications in certain real-world machine learning scenarios,
as well as its apparent difficulty, we revisit the learning with attribute noise model and study it under 
product random attribute noise, in which the noise rates are {\em not} known. We overcome the information-theoretic impossibility result of \cite{Sloan88,GoldmanS95} by allowing the algorithm output a small {\em list} of labeling functions that contains one which is accurate. Thus, even if it is impossible to identify a single accurate function, we can hope to produce a small list of candidate hypotheses that contains an accurate one. Indeed, the proof of~\cite{Sloan88, GoldmanS95} follows from an explicit construction of two pairs $(\calD_1, c_1, \calR_1)$ and $(\calD_2, c_2, \calR_2)$ of distributions, distinct dictators as labeling functions, and product noise distributions, respectively. The two pairs of tuples lead to exactly the same observed distribution over the $n+1$ bits received $(\tilde{x}, c(x))$, when $\nu>2\epsilon$, where $\nu$ is an upper bound  on the noise amount per attribute. In the list-learning model the algorithm is allowed to output both solutions. In fact,  as in PAC learning, any $(1-\eps)$-accurate hypothesis with respect to the input distribution $\calD_i$ is a valid solution to the learning problem, hence it is enough to outputs a small {\em net} of hypotheses that covers all the valid inputs, in the sense that for any valid input that could have resulted in the observed distribution, the list contains a hypothesis that is $(1-\eps)$-accurate with respect to that input.

Our results provide some sufficient conditions where efficient list learning is still possible despite the previous barriers. We also show strong lower bounds for most natural classes of Boolean functions.

\subsection{The model: list learning with attribute noise}

We denote by an instance of the attribute learning problem to be a tuple $(\calD, c, \calR)$, where $\calD$ is the unknown distribution from which the algorithm receives noisy samples,  $c$ is the labeling function, and $\calR$ is the noise distribution. 
We will denote by $\tilde{\calD}$ the observed distribution of $(\tilde{x}, c(x))$, where $\tilde{x}=x+\rho$, and $x\leftarrow \calD$ and $\rho\leftarrow \calR.$ We will often abuse notation and denote the marginal distribution on $\tilde{x}$ by $\tilde{\calD}$ as well.

For an observed distribution $\tilde{\calD}$, a {\em net} $\calH$ (specifically, an {\em $\eps$-net}) is a set of $(1-\eps)$-accurate solutions such that for any tuple $(\calD, c, \calR)$ that could have resulted in the observed distribution $\tilde{\calD}$, there exists  $h\in \calH$ that is a $(1-\eps)$-accurate solution with respect to $c$ and $\calD$. 

Inspired by the list-decoding model in coding theory, we seek answers to the following general questions:
\begin{enumerate}
\item (Combinatorial): Does there exist a small net $\calH$ for the attribute noise learning problem with observed distribution $\tilde{\calD}$?
\item (Algorithmic): Can a net for the attribute noise learning problem with observed distribution $\tilde{\calD}$ be computed efficiently?
\end{enumerate}

We formalize these notions below, in the attribute-noise PAC-learning model, with product random noise.

\begin{definition}(List learning with random product attribute noise)
Let $\calC$ be a concept class containing Boolean functions $c:\{0,1\}^n\rightarrow \{0,1\}$, $\calD$ a distribution over $\{0,1\}^n$,  let 
 $\nu, \epsilon \in (0,1)$, and  $0\leq p_1, \ldots, p_n \leq \nu$.
 Let $\calR$ be noise distribution defined as the product of $n$ independent Bernoulli distribution with parameters $p_i$, $i\in [n].$

 \begin{enumerate}
     \item (Combinatorial) $\calC$ is said to be {\em list-learnable} with list size $\ell=\ell(\nu, \eps)$ if there exists a net $\calH$ for the solutions of the attribute noise learning problem with input distribution $\calD$, such that $|\calH|\leq \ell$. 
     \item (Algorithmic) $\calC$ is said to be {\em algorithmically list learnable} if there exists a randomized algorithm outputting 
 all $h\in \calH$  with probability $1-\delta$ in time proportional to $\ell$.
 \end{enumerate}

 \end{definition}

\subsection{Our results}

First, we show that the classes of parities and majorities are not amenable to efficient list learning, as every net for them has exponential size, regardless of the representation used for the net. More generally, we obtain our lower bound for any symmetric family of functions with sufficiently high {\em noise sensitivity}. (Recall that the noise sensitivity under $\rho$ noise, $\NS_{\rho}(f)$, is the probability the value of $f$ changes when its inputs are corrupted by product noise of rate $\rho$.)

\begin{theorem}\label{thm:symmetric:inf}
(Theorem~\ref{thm:symmetric}, informal)
Let $f$ be a   symmetric function $f:\{0,1\}^{n/2}\rightarrow \{0,1\}$. Let $\calF_f$ be the family of functions  on $n$ bits containing all functions $f_S$ obtained by instantiating  $f$ on the set $S\subset [n]$ with $|S|=n/2$. 
Let $\rho>0$. Suppose $\epsilon\leq (\frac12-o(1))\NS_{\rho/15}(f)$.  Then 
if for every $f_S\in \calF_f$ and distribution $\calD$ on $\bx$ there is an $h\in\calH$ satisfying $\Pr_{\bx \sim \calD}[f_S(x) \neq h(x)]<\eps$, then $|\calH|>2^{\Omega(n)}.$
\end{theorem}

Two immediate corollaries follow:

\begin{corollary}
Taking $f(x_1, x_2, \ldots, x_{n/2})=\sum_{i=1}^{n/2} x_i$, namely $f = \mathrm{PARITY}_{n/2}$, 
in Theorem \ref{thm:symmetric:inf}, the lower bound holds for any $\rho>0$ and $\eps<\frac14-o(1).$ 

\end{corollary}

\begin{corollary}
Taking $f(x_1, x_2, \ldots, x_{n/2})=\mathrm{MAJORITY}({x_1, x_2, \ldots, x_{n/2}})$, namely $f = \mathrm{MAJORITY}_{n/2}$, 
in Theorem \ref{thm:symmetric:inf}, the lower bound holds for any $\rho > 0$ and $\eps<\Omega(\sqrt \rho).$ 

\end{corollary}

We stress that since these lower bounds hold regardless of the representation used in the list, they give lower bounds for richer function classes that contain parities or majorities (respectively) as special cases, such as general linear threshold functions and so on. Of course, such a distinction between ``proper'' (representation-specific) and ``improper'' (representation-independent) solutions does not arise in coding theory, but is a common feature in learning theory. Improper learning is the main subject of interest in learning theory, but lower bounds against improper learning algorithms are usually much more challenging. The same holds here: it is generally much easier to argue that an exponential lower bound holds if the function is forced to be a parity function or a conjunction (see below), for example. 

\ignore{
\begin{theorem}
There exists a family $\mathcal{F}$ of functions $f:\{0,1\}^n \rightarrow \{0,1\}$, such that for any $0<\epsilon<1/4-o(1)$, and attribute noise $\rho>\eps$, any net $\calH$ of functions satisfying 
\[
\max_{z \in \{0,1\}^{n/2}} \min_{h \in \calH} \Pr_{\bx \sim \calD^z}[f^z(\bx) \neq h(\bx)]
<\epsilon\]
must have $|\calH|>2^{\Omega(n)}.$
 \end{theorem}
 }

Our main results focus on conjunctions, for which we give a general lower bound, and an upper bound for a specific restriction on the input distribution on examples.

\begin{theorem}\label{thm:lb-conjunctions}
(Theorem~\ref{thm:conj:lower}, informal)
Let $k>0$ be an integer, $\epsilon>0$, and let  $\mathcal{C}_k$ be the set of all conjunctions over $k$ bits out of $n$ bits $f:\{0,1\}^n \rightarrow \{0,1\}$. If the attribute noise is $\rho=\frac1k>8\epsilon$, 
then there is an input distribution $\calD$ such that list learning $\mathcal{C}_k$ under $\calD$
with accuracy $\eps$ would require a list of size
$|\calH|>2^{\Omega(k)}.$

\end{theorem}

Again, since this theorem is representation-independent, we obtain the same lower bound for any family of functions that can express the conjunctions on $k$ out of $n$ bits. Thus, even with $k=\Omega(n)$, we obtain lower bounds for decision trees, DNFs, $s$-CNFs, and so on. (By standard reductions, i.e., swapping $0$ and $1$, one can also obtain the same lower bound for $s$-DNFs.) Between Theorem~\ref{thm:symmetric:inf} and the above, we have lower bounds for essentially all of the natural families of functions studied in learning theory, provided that the function depends on $\omega(\log n)$ coordinates. (When $k=O(\log n)$, the problems are all open, see Section~\ref{intro:op}.)

Our main result is a sufficient assumption on the input distribution on examples that allows efficient list learning of sparse conjunctions under arbitrary probabilities of flipping individual attributes.

\begin{theorem}
(Theorem~\ref{thm:learning_main}, informal)
For any positive integer $k, k'$, and any real number $0< \eps, \delta <1$, $0<\gamma\leq 1/2$, 
there exists a randomized algorithm which, 
with probability at least $1-\delta$, list learns $k$-conjunctions 
with accuracy $1-\eps$, with sample complexity $\poly{k, \frac{1}{\eps}, \frac{1}{\gamma}, \log{\frac{1}{\delta}}}$
and time complexity $\poly{n, \frac{1}{\eps}, \log{\frac{1}{\delta}}, \frac{1}{\gamma}, (\frac{k}{\eps \gamma})^k}$
in the attribute-noise model with bit noise rate 
$0\leq \nu_i<\frac12-\gamma$ for every $1\leq i \leq n$,
under the assumption that the ground-truth distribution is $k'$-wise independent.
\end{theorem}

We note that the trivial PAC learning algorithm that tries all monotone conjunctions of size at most $k$ works only for noise rate $\nu\leq \frac{\eps}{2k}$ -- we include the proof for completeness in the Appendix \ref{sec:trivialPAC}.

\subsection{Further discussion of related work}\label{sec:relwork}

 The information theoretic lower bounds of \cite{Sloan88, GoldmanS95} are analogous to the classical scenario in coding theory, in which, upon receiving a word corrupted by a high amount of noise, decoding becomes ambiguous. As a result, Elias \cite{elias} and Wozencraft \cite{wozencraft} extended the classical notion of unique decoding to that of {\em list-decoding}, where the algorithm is required to output a list of all possible messages that could have resulted in the received one.
A similar motivation prompted Balcan, Blum and Vempala \cite{BalcanBV08} to introduce the notion of {\em list-decodable learning} in the context of clustering, where their algorithm is required to output a small list that includes a ``good'' clustering, with high probability. Follow-up results by Charikar, Steinhardt and Valiant \cite{CharikarSV17}  use this framework in the context of learning from untrusted data when there is a minority fraction of ``inliers'' and so identifiability cannot hold. In the same vein,  Diakonikolas, Kane and Stewart \cite{DiakonikolasKS18} obtain algorithms for robust mean estimation, and learning mixtures of Gaussians. More recently, Karmalkar, Klivans, and Kothari \cite{karmalkar2019list} and Raghavendra and Yau \cite{raghavendra2019list} independently gave list-decodable linear regression algorithms for this minority-inlier setting. In all of these works, the difference is that there is guaranteed to be a fixed fraction of uncorrupted examples (whereas the corruption of the remaining examples is arbitrary). By contrast, in the attribute-noise model we study, with high probability {\em every} example has a non-negligible fraction of corrupted attributes, though conversely, the corruptions are stochastic and independent. Nevertheless, in spite of ours being a stochastic-noise model, we will see that the lack of clean examples still poses serious challenges, even for a list learner.

\subsection{Highlights of techniques}

\paragraph{The lower bounds.}

The high-level idea of the lower-bound proofs is to explicitly construct a large set of labeling functions $c\in \calC$ and initial input and noise distributions such that any function  in the net can only be $(1-\eps)$-accurate for a small number of possible initial solutions $(\calD, c, \calR)$, regardless of the representations used for the functions in the net. Hence, to cover an exponential number of such potential solutions a net has to have large size. The construction of the initial distributions exploits the idea that bits $(x_{2i}, x_{2i+1})$  that are $\rho$-{\em correlated} (meaning that $x_{2i+1}$ takes the same value as $x_{2i}$ w.p. $1-\rho$, and takes the flipped value with probability $\rho$) appear identical to an observer  when adding Bernoulli random noise $\rho$ to one copy and no noise to the other copy. In the cases of families of  majorities, and of parity functions, we exploit this observation together with the fact that  totally symmetric functions with high ``noise sensitivity" are often far apart. Thus, any single member of the net can only be accurate for at most one of these far pairs, and so we must have a large net.

\paragraph{The upper bounds.} 
The essential difficulty in learning conjunctions under the attribute noise model is that on the one hand, conjunctions are in general very sensitive to the attributes that appear in them; missing even one significant attribute incurs a large error. But, on the other hand, as illustrated in the lower bound, it is in general impossible to distinguish bits of the conjunction corrupted by noise in our examples from bits that would thus incur a serious error if they were included in the conjunction. Thus, we seek to find a small set of candidate coordinates and output all small subsets of these. Both the size of the set of candidates and the size of the conjunctions must be small to obtain a polynomial-size list. Proving that the algorithm does output a net for the solution space is the most difficult part of our arguments, the difficulty emerging from the fact that the accuracy of the solution is measured against the original unknown distribution rather than the observed distribution itself. The algorithm can only perform tests and optimize quantities using the corrupted examples, and we must then bound the distances from the unknown distribution.

The algorithm for list learning conjunctions under random attribute product noise operates under the assumption that the attributes in the initial distribution on examples are pairwise independent. We first observe that since the bits of the actual conjunction must all take value $1$ on label $1$, and the noise is a product distribution, the bits of the actual conjunction in the noisy examples are fully independent when conditioned on label $1$. The algorithm thus first identifies the subset of variables that are (at least) pairwise independent on label $1$, and then eliminates from this surviving set the variables that are not too sensitive to the label. These eliminated variables could not have been significant bits of the conjunction: if there is no attribute noise, the variables in the conjunction would be very sensitive to the label, since they would always take value $1$ on label $1$, and they would take value $0$ on label $0$ significantly often. Now, either the function is nearly constant and so a constant function predicts the label sufficiently well, or else there is a bounded statistical distance between the distribution conditioned on label $1$ and the original distribution, which is a mixture of the label $1$ and label $0$ distributions. We show that when the function is far from constant, there cannot be too many coordinates surviving. Intuitively, otherwise, the weight would allow us to distinguish the label $1$ distribution from the original distribution beyond the statistical distance, due to Chebyshev's inequality: the total weight would concentrate if there were many coordinates left. Thus we can afford to enumerate all small subsets of the surviving coordinates in this case.

\subsection{Open Problems}\label{intro:op}

Our results seek to bring forth the natural, yet difficult-to-analyze model of learning under attribute noise. While we prove several impossibility results and a sufficient condition for learning sparse conjunctions, our work leaves open a plethora of intriguing possibilities. We describe below a few important ones. 

The first, most natural question is whether or not the pairwise-independence assumption is really needed for our algorithm:

\begin{question}
Is the set of sparse conjunctions list-learnable under arbitrary product distributions of the attribute noise?
\end{question}

\noindent
But, moreover, we note that our lower bounds do not rule out the possiblity of obtaining polynomial-size lists for $O(\log n)$-sparse functions in general. So it is still open whether or not natural function families with small numbers of relevant coordinates have efficient list-learning algorithms, e.g.:

\begin{question}
Is the set of sparse Boolean threshold functions list-learnable under arbitrary product distributions of the attribute noise?
\end{question}

Thus, in contrast to the usual theory of supervised learning, we do not have a characterization of which families of functions are (information-theoretically) learnable in terms of some parameter like the VC-dimension or Rademacher complexity in the attribute noise list-learning setting:

\begin{question}
What are necessary and sufficient conditions for families of Boolean functions to be list-learnable under the product distribution of the attribute noise?
\end{question}

\noindent
Or, more generally:

\begin{question}
What families of Boolean functions are list-learnable under general (not-necessarily independent product) noise distributions?
\end{question}

\noindent
Of course, one can ask both computational/algorithmic and statistical/combinatorial variants of these questions. But again, a central difficulty here is that the usual statistical techniques for estimating losses from data cannot be used directly to estimate losses from our corrupted data. Thus it seems that new tools may need to be developed to address these questions.

\section{Lower Bounds}

\subsection{Noise sensitivity lower bound for some  symmetric functions} \label{sec:lb-totallysym}

In this section we show that some families of  symmetric functions on subsets of half the bits are hard to improperly learn in an information-theoretic sense, and prove Theorem \ref{thm:symmetric:inf}.

\ignore{
\begin{theorem}

There exists a family $\mathcal{F}$ of functions $f:\{0,1\}^n \rightarrow \{0,1\}$, such that for any $0<\epsilon<1/4-o(1)$, and attribute noise $\rho>\eps$, any net $\calH$ of functions satisfying 

\[
\max_{z \in \{0,1\}^{n/2}} \min_{h \in \calH} \Pr_{\bx \sim \calD^z}[f^z(\bx) \neq h(\bx)]
<\epsilon\]

must have $|\calH|>2^{\Omega(n)}.$
 
\end{theorem}

In fact we prove a more general result about the list learnability of classes of functions $f$ obtained by extending symmetric functions on $n/2$  bits that have high sensitivity  to functions on $n$ bits,  as described in the remainder  of the subsection.
}

Before defining the functions in $\mathcal{F}$, we will make some notational conventions. For the sake of presentation we assume $n$ is even.

For a string $x\in \{0,1\}^n$, we may view it as the concatenation of pairs $(x_{2i+1}, x_{2i+2})$, for $i=0, 1, ...,n/2-1$, and define two strings $x^0, x^1 \in \{0,1\}^{n/2}$, by selecting the odd, respectively the even, indices of these pairs in order, namely $x^0=x_1, x_3, \ldots, x_{n-1}$ and $x^1=x_2, x_4, \ldots, x_n$.
For $x\in \{0,1\}^n$ and a string $z\in \{0,1\}^{n/2}$, we define the hybrid string $x^{z}\in \{0,1\}^{n/2}$ to be the string that for each $0\leq i\leq n/2-1$ selects either $x_{2i+1}$ if $z_i=0$, or $x_{2i+2}$ if $z_i=1$,  denoted by $x^{z}= (x_1^{z_1}, x_2^{z_2}, \ldots, x_{n/2}^{z_{n/2}})$, where $x_i^{z_i}=x_{2i+1}$ if $z_i=0$, and $x_i^{z_i}=x_{2i+2}$ if $z_i=1.$

We now define the set of functions $\mathcal{F}$. For a symmetric function $f : \{0,1\}^{n/2} \to \{0,1\}$, such as {\em parity} or {\em majority}, and a string $z \in \{0,1\}^{n/2}$, let $f^{z} : \{0,1\}^n \to \{0,1\}$ be the function $f^{z}(x) = f(x^{z}) = f(x_1^{z_1},x_2^{z_2},\ldots,x_{n/2}^{z_{n/2}})$. Let $$\mathcal{F}=\mathcal{F}(f)=\{f^{z}\}_{{z}\in \{0,1\}^{n/2}}.$$

Further, for $z\in \{0,1\}^{n/2}$ let $\calD^{z}$ be the distribution\footnote{Actually, $\calD^{z}$ is the same distribution no matter what $z$ is.} on $\{0,1\}^n$ defined by the following probability experiment:

\begin{itemize}
	\item The coordinates in ${\bx}^{z}$ are drawn independently and uniformly at random.  That is, ${\bx}^{z} \sim \calU_{n/2}$, where $\calU_{n/2}$ represents the uniform distribution on $\{0,1\}^{n/2}$.
	\item The coordinates in $\bx^{\overline{z}}$ are $\rho$-noisy copies of ${\bx}^{z}$ ; specifically, each bit $x_i^{\overline{z_i}}$ is a $\rho$-noisy copy of $x_i^{z_i}$.
\end{itemize}

We will show that if $z$ is unknown, and we see labeled examples according to $f^{z}$ under $\calD^{z}$ with $\rho$-bounded attribute noise, then list-learning to small accuracy requires an exponential size list.  That is, for every set of functions $\calH$ (our proposed net), the quantity

\[
\max_{z \in \{0,1\}^{n/2}} \min_{h \in \calH} \Pr_{\bx \sim \calD^{z}}[f^{z}(x) \neq h(x)]
\]
is ``large'' if $|\calH|$ is sub-exponential.

For $f^z$ with respect to $\calD^z$, given $x$, the attribute noise $N^z_{\rho}(x)$ is as follows:  we apply $\rho$-noise to each $x_i^{z_i}$, and no noise to $x_i^{\overline{z_i}}$.  It follows that for \emph{every} $\calD^z$, the resulting distribution over the \emph{labeled} examples is the same.  We define $\calD$ to be distribution\footnote{Actually, this is the same as $\calD^z$.} on $\{0,1\}^n$ such that, for each $i$, $x^0_i$ and $x^1_i$ are $\rho$-correlated uniformly random bits, and the $n/2$ pairs ($x^0_i,x^1_i$) are chosen independently.  It can be easily checked that the distribution $\calD$ has the following properties:

\begin{itemize}
	\item For every $z \in \{0,1\}^{n/2}$ and a random string $\bx \sim \calD$, $\bx^z$ is distributed as a uniformly random string over $\{0,1\}^{n/2}$.
	\item For every pair of strings $z,z' \in \{0,1\}^{n/2}$ and a random string $\bx \sim \calD$, the random strings $\bx^z$ and $\bx^{z'}$, restricted to the coordinates where $z$ and $z'$ disagree, are $\rho$-noisy copies of each other.
	\item To construct the distribution of $\bx^{z'}$ from $\bx^z$, one can apply $\rho$-noise to the coordinates of $\bx^z$ in those coordinates where $z$ and $z'$ differ (and just read off the coordinates of $\bx^z$ where they are the same). 
	\item In fact, $\calD^z$ is identical to $\calD$ for every $z \in \{0,1\}^{n/2}$.  However, the distribution of labeled examples $\langle \bx, f^z(\bx) \rangle$ where $\bx \sim \calD^z$ depends on $z$.  The distribution of labeled examples after attribute noise $\langle N^z_{\rho}(\bx), f^z(\bx) \rangle$ is independent of $z$; the marginal distribution on $N^z_{\rho}(\bx)$ is $\calD = \calD^z$.
\end{itemize}

\subsubsection{Noise sensitivity}
Recall that the {\bf noise operator at $\rho$ on $S$}
	is denoted by $N_{S,\rho}(x)$ is a random string such that $N_{S,\rho}(x)_i$ is a uniform random bit $\rho$-correlated with $x_i$ if $i \in S$, and $N_{S,\rho}(x)_i = x_i$ with probability $1$ for $i \notin S$.
	The {\bf noise sensitivity at $\rho$ on $S$} to be $\NS_{S,\rho}(f) = \Pr_{\by \sim \calU_{n/2}}[f(\by) \neq f(N_{S,\rho}(\by)]$.  These are related to the standard noise sensitivity constructions via $N_{\rho}(x) = N_{[n],\rho}(x)$, and $\NS_{\rho}(f) = \Pr_{\by \sim \calU_{n/2}}[f(\by) \neq f(N_{\rho}(\by))]$ (cf.\ \cite{ref:OD14}).
		
	\begin{claim} 
		Let $S \subseteq [n]$ be a set such that $|S| = n/14$.  For every symmetric Boolean function $f$ on $n/2$ variables such that $\NS_{S,\rho}(f) = 2^{-o(n)}$ for all $S$, $\NS_{S,\rho}(f) \geq (1-o(1))\NS_{\rho/15}(f)$.  
	\end{claim}

	\begin{proof}
		Note that, for every $x$, $N_{\rho/15}(x)$ is distributed as $N_{\bT,\rho}(x)$, where $\bT$ is a set where each coordinate is included independently with probability $1/15$.
		It follows that

		\begin{align*}		
		\NS_{\rho/15}(f) &= \Pr_{\by \sim \calU_{n/2}}[f(\by) \neq f(N_{\rho/15}(\by))] \\
		&= \Pr_{\by \sim \calU_{n/2}}[f(\by) \neq f(N_{\bT,\rho}(\by))] \\
		&= \E_{\bT}[\NS_{\bT,\rho}(f)].
		\end{align*}

		By a Chernoff bound, $\Pr[|\bT| \leq n/14] \geq 1 - 2^{-\Omega(n)}$.  Thus, for a set $S$ such that $|S| = n/14$, we have
		\begin{align*}
		\NS_{\rho/15}(f) & = \E_{\bT}[\NS_{\bT,\rho}(f)] \\
		& = \E_{\bT}[\NS_{\bT,\rho}(f) \mid |\bT| \leq n/14]\Pr_{\bT}[|\bT| \leq n/14] + \E_{\bT}[\NS_{\bT,\rho}(f) \mid |\bT| > n/14]\Pr_{\bT}[|\bT| > n/14]
		\\
		& \leq \NS_{S,\rho}(f) \Pr[|\bT| \leq n/14] + 2^{-\Omega(n)} \\
		& \leq \NS_{S,\rho}(f) (1+o(1)), \\
		\end{align*}
		where we used the fact that $\NS_{S,\rho}$ is nondecreasing as $|S|$ increases.  (Since we assumed that $f$ is symmetric, only $|S|$ matters.)
	
		Dividing both sides by the $(1+o(1))$ factor yields the claim.
	\end{proof}

\begin{lemma}
	Let $z,z' \in \{0,1\}^{n/2}$ be strings such that $|z - z'| \geq n/14$.  Then $\Pr_{\bx \sim \calD^z}[f^z(\bx) \neq f^{z'}(\bx)] \geq (1-o(1))\NS_{\rho/15}(f)$.
\end{lemma}

\begin{proof}
	Define $S$ to be the set of strings where $z$ and $z'$ differ.
\begin{align*}
	\Pr_{\bx \sim \calD^z}[f^z(\bx) \neq f^{z'}(\bx)] &=
	\Pr_{\bx \sim \calD^z}[f(\bx^z) \neq f(\bx^{z'})] \\
	&= \Pr_{\bx \sim \calD^z}[f(\bx^z) \neq f(N_{S,\rho}(\bx^{z}))] \\
	&= \Pr_{\by \sim \calU_{n/2}}[f(\by) \neq f(N_{S,\rho}(\by))] \\
	&= \NS_{S,\rho}(f) \\
	&\geq (1-o(1))\NS_{\rho/15}(f).
\end{align*}
\end{proof}	

We finally prove a more specific version of Theorem \ref{thm:symmetric:inf}.

\begin{theorem}\label{thm:symmetric}
Let $f:\{0,1\}^{n/2} \rightarrow \{0,1\}$ be a  symmetric function, and $\rho>0$. If $\epsilon\leq (\frac12-o(1))\NS_{\rho/15}(f)$ then, for family $\calF=\{f^z\}_{z\in \{0,1\}^{n/2}}$ of Boolean functions on $n$ bits where the oracle produces examples with attribute noise rate $\rho$, we have that any net $\calH$ satisfying

\[
\max_{z \in \{0,1\}^{n/2}} \min_{h \in \calH} \Pr_{\bx \sim \calD^{z}}[f^{z}(x) \neq h(x)]<\eps
\]

must have $|\calH|>2^{\Omega(n)}.$
\end{theorem}

\begin{proof}
By the triangle inequality, no function in the net can approximate
both $f^z$ and $f^{z'}$ for two strings $z,z'$ where $|z - z'| \geq n/14$ (with respect to $\calD = \calD^z = \calD^{z'}$) to within $(\frac12-o(1))\NS_{\rho/15}(f))$.  Thus, any  function in the net can cover at most $\binom{n}{n/14}$ such functions $f^z$ with respect to $\calD^z$.  It follows that any net requires $2^{n/2}/\binom{n}{n/14} \geq 2^{n/14}$ functions (here we used that ${n\choose k} < (ne/k)^k$, with $k=n/14$).

\end{proof}

\begin{remark}
    The symmetric assumption can be relaxed by noting that the bound works for any function that is roughly balanced over the uniform distribution, since the noise sensitivity of such functions is $\Omega(\min\{\Pr[f(\bx) = 0],\Pr[f(\bx) = 1]\})$.  Roughly speaking, this result asserts that we cannot learn with error smaller than the noise sensitivity.
\end{remark}
\subsection{Maximum sensitivity lower bound for conjunctions}

In this section we show a lower bound for improper list learning of conjunctions and by proving a more specific version of Theorem \ref{thm:lb-conjunctions}.
 We will use the same notation as in Section \ref{sec:lb-totallysym}.
 
\begin{theorem} \label{thm:conj:lower}
Let $k>0$ be an integer, $\epsilon>0$, and let  $\mathcal{C}_k$ be the set of all conjunctions over $k$ bits out of $n$ bits $f:\{0,1\}^n \rightarrow \{0,1\}$. If the attribute noise is $\rho=\frac1k>8\epsilon$, then any net $\calH$  of functions satisfying 

\[
\max_{z \in \{0,1\}^{n/2}} \min_{h \in \calH} \Pr_{\bx \sim \calD^z}[f^z(\bx) \neq h(\bx)]
<\epsilon\]

must have $|\calH|>2^{\Omega(k)}.$
 
\end{theorem}

\begin{proof}
Suppose that the distribution $\calD^z$ over $\{0,1\}^{2k}$ is such that

\begin{itemize}
	\item The coordinates in $\bx^z$ are drawn independently at random with bias $1/k$.  That is, $\bx^z \sim \mu_{k,1/k}$, where $\mu_{n,p}$ denotes the $p$-biased distribution over $\{0,1\}^n$.
	\item The coordinates in $\bx^{\overline{z}}$ are $\rho$-noisy copies of $\bx^z$; specifically, each bit $\bx_i^{\overline{z_i}}$ is a $\rho$-noisy copy of $\bx_i^{z_i}$.
\end{itemize}

We will show that if $z$ is unknown, and we see labeled examples according to $f^z$ under $\calD^z$ with $\rho$-bounded attribute noise, then list-learning to small accuracy requires an exponential size list.  That is, for every set of functions $\calH$ (our proposed net), the quantity

\[
\max_{z \in \{0,1\}^{n/2}} \min_{h \in \calH} \Pr_{\bx \sim \calD^z}[f^z(\bx) \neq h(\bx)]
\]
is ``large'' if $|\calH|$ is sub-exponential in $k$.

For $f^z$ with respect to $\calD^z$, given $x$, the attribute noise $N^z_{\rho}(x)$ is as follows:  we apply $\rho$-noise to each $x_i^{z_i}$, and no noise to $x_i^{\overline{z_i}}$.  It follows that for \emph{every} $\calD^z$, the resulting distribution over the \emph{labeled} examples is the same.  We define $\calD$ to be distribution\footnote{Actually, this is the same as $\calD^z$.} on $\{0,1\}^n$ such that, for each $i$, $x^0_i$ and $x^1_i$ are $\rho$-correlated random bits with bias $(1-\rho)(1/k)+\rho(1-1/k)$, and the $k$ pairs ($x^0_i,x^1_i$) are chosen independently.  It can be easily checked that the distribution $\calD$ has the following properties:

\begin{itemize}
	\item For every $z \in \{0,1\}^{n/2}$ and a random string $\bx \sim \calD$, $\bx^z$ is distributed as a uniformly random string over $\{0,1\}^{n/2}$.
	\item For every pair of strings $z,z' \in \{0,1\}^{n/2}$ and a random string $\bx \sim \calD$, the random strings $\bx^z$ and $\bx^{z'}$, restricted to the coordinates where $z$ and $z'$ disagree, are $\rho$-noisy copies of each other.
	\item To construct $\bx^{z'}$ from $\bx^z$, one can apply $\rho$-noise to the coordinates of $\bx^z$ in those coordinates where $z$ and $z'$ differ (and just read off the coordinates of $\bx^z$ where they are the same). 
	\item In fact, $\calD^z$ is identical to $\calD$ for every $z \in \{0,1\}^{n/2}$.  However, the distribution of labeled examples $\langle \bx, f^z(\bx) \rangle$ where $\bx \sim \calD^z$ depends on $z$.  The distribution of labeled examples after attribute noise $\langle N^z_{\rho}(\bx), f^z(\bx) \rangle$ is independent of $z$; the marginal distribution on $N^z_{\rho}(\bx)$ is $\calD = \calD^z$.
\end{itemize}

Unlike the uniform distribution case, when we consider the accuracy of a function in the net on a conjunction, the distribution under which we calculate the error depends on the conjunction.  We compute the following quantities first:

\begin{itemize}
	\item The probability of the all-$0$'s string in the true distribution is $(1-1/k)^k(1-\rho)^k$; the all $0$'s string in drawn in the conjunction bits, and no flips occur in the noisy version.
	\item The probability of a string of all-$0$'s, except for $x_i^b = 1$ depends on the conjunction.  If $z_i = b$ ($x_i^b$ is in the conjunction),
	then the probability mass assigned is $(1-1/k)^{k-1}(1/k)(1-\rho)^{k-1}\rho$.  If $z_i = 1-b$ ($x_i^b$ is not in the conjunction), then the probability mass assigned is $(1-1/k)^k(1-\rho)^{k-1}\rho$.
\end{itemize}

Consider the values of a function $f$ on these standard basis strings.  

\begin{itemize}
	\item If $f(e_{i,b}) = 1$ ($x_i^b = 1$) and $z_i = b$ ($x_i^b$ is in the conjunction), $f$ incorrectly computes the conjunction.  The contribution to the error is $(1-1/k)^{k-1}(1/k)(1-\rho)^{k-1}\rho$.
	\item If $f(e_{i,1-b}) = 0$ ($x_i^b = 0$) and $z_i = b$ ($x_i^b$ is in the conjunction), $f$ incorrectly computes the conjunction.  The contribution to the error is $(1-1/k)^k(1-\rho)^{k-1}\rho$.    
\end{itemize}

So for every conjunction, a false $0$ is roughly $k$ times as costly as a false $1$.  To make the error less than $(1-1/k)^{k-1}(1/k)(1-\rho)^{k-1}\rho \cdot (99k/100)$, there must be a function in the net that has no false $0$'s and at most $99k/100$ false $1$'s on these strings.  A function in the net covers the most conjunctions by taking $f$ to be $1$ on $k+99k/100=199k/100$ of these strings and $0$ on the other $k/100$.  A function is covered if its bits are correspond to those with ones.  There are $2^{99k/100}$ conjunctions covered, but $2^k$ conjunctions in total, so any net must have $2^{k/100}$ functions in it to achieve error below $(1-1/k)^{k-1}(1/k)(1-\rho)^{k-1}\rho \cdot (99k/100)$.  Taking $\rho = 1/k$, this is at least
\begin{align*}
(1-1/k)^{k-1}(1/k)(1-1/k)^{k-1}(1/k) \cdot (99k/100) 
&= 99(1-1/k)^{2k-2}/(100k) \\
&\geq 99/(100e^2k) \\
&\geq 1/(8k),
\end{align*}
so the error is at least $\rho/8$.  We need $\rho < 8\eps$ for a sub-exponential size net.
\end{proof}

\section{Upper Bounds}
\subsection{Definitions and some basic facts}

We use the following notation:
\begin{itemize}
    \item $\tilde{D}$: the observed distribution
    \item $D$:  the original distribution before applying the attribute noise
	\item $c=\land_{i\in c}\ell_i$: a conjunction\footnote{We abuse notation here to let $c$ denote both the conjunction 
	and the set of variables in the conjunction. Furthermore, the conjunction over the empty set is understood to be
$\mathbf{1}$.} of size at most $k$, where $c\subset [n]$, $|c|\leq k$ and $\ell_i$ is either
	$x_i$ or $1-x_i$
	\item $D_b$ (resp. $\tilde{D}_b$): the original (resp. observed) distribution conditioned on label $c$ being $b$, for $b\in \{0,1\}$ 
    \item $\nu_i$: the attribute noise rate of bit $i$
\end{itemize}

We call a bit $i\in [n]$ a \emph{conjunction bit} if $i\in c$ and \emph{non-conjunction bit} otherwise.
Note that without loss of generality, we may assume that every candidate conjunction bit in $S$ is biased towards $1$, i.e.
$\E_{\tilde{D}}[x_i]\geq 1/2$ for every $i\in S$, as otherwise we simply replace $x_i$ with $1-x_i$ in our arguments.

\begin{definition}[Non-uniform $k$-wise independence]
Let $P:\{0,1\}^n \to \R^{\geq 0}$ be a distribution and $k$ be a positive integer.
$P$ is said to be \emph{(non-uniform) $k$-wise independent} if
for any subset of $k$ indices $\{i_{1}, \ldots, i_{k}\} \subset [n]$ and for any 
$z_{1}\ldots z_{k}\in \{0,1\}^{k}$,
\[
\Pr_{P}[X_{i_{1}}\cdots X_{i_{k}}=z_{1}\cdots z_{k}]
=\Pr_{P}[X_{i_{1}}=z_{1}] \times \cdots \times \Pr_{P}[X_{i_{k}}=z_{k}].
\]
\end{definition}

\begin{claim}\label{claim:noise_indep}
For any positive integer $k$ and any distribution $D:\{0,1\}^n \to \R^{\geq 0}$, 
$D$ is $k$-wise independent if and only $\tilde{D}$ is $k$-wise independent.
In other words, attribute noise does not change the $k$-wise independence of the underlying distribution.
\end{claim}
We defer the proof of this Claim to  Appendix  \ref{proofclaim:kwise}.

Learning conjunctions is easy when there is no attribute noise because, if $x_i$ is in the conjunction,
then conditioned on label being $1$, $\Pr[X_i=1]=1$ and this probability should be lower
without the conditioning --- unless variable $x_i$ is almost surely being $1$ under the distribution $D$.
In other words, the expectation of a (relevant) conjunction bit should be \emph{sensitive} to
label change. This is also true under attribute noise, although with lower sensitivity in general.

\begin{definition}
The (observed) \emph{label sensitivity} at bit $i$ is defined by 
$\mathrm{LS}_i=\E_{\tilde{D}_1}[X_i]-\E_{\tilde{D}_0}[X_i]$; 
that is, $\mathrm{LS}_i$ is the difference between expectation of $x_i$ conditioned on label being $1$ 
and the expectation of $x_i$ conditioned on label being $0$.
\end{definition}

Finally we note the following simple fact: since attribute noise does not change the labels of examples, the total mass of positive or negative examples are the same for $D$ and $\tilde{D}$.
\begin{fact}\label{fact:eps_mass}
For any underlying distribution $D$ of the example oracle and any attribute noise vector $\mathbf{\nu}$,
$\Pr_{D}[c(x)=1]=\Pr_{\tilde{D}}[c(x)=1]$ and $\Pr_{D}[c(x)=0]=\Pr_{\tilde{D}}[c(x)=0]$. 
\end{fact}

\subsection{Main theorem on learning conjunctions when the underlying distribution is $k'$-wise independent}
Our main theorem of this section is the following
\begin{theorem}\label{thm:learning_main}
For any positive integer $k$ and any real numbers $0< \eps, \delta <1$, $0<\gamma\leq 1/2$, 
there exists a randomized algorithm which, 
with probability at least $1-\delta$, list-learns $k$-conjunctions 
with accuracy $1-\eps$, with sample complexity $\tilde{O}(k^{4}\log(1/\delta)/(\eps^{9}\gamma^{4}))$
and time complexity 
$\max\{\tilde{O}(n^{2}k^{4}\log(1/\delta)/(\eps^{9}\gamma^{4})), O((32k^{2}/\eps^{5}\gamma^{2})^{k})\}$, 
in the attribute-noise model with bit noise rate 
$0\leq \nu_i<\frac12-\gamma$ for every $1\leq i \leq n$,
under the assumption that the ground-truth distribution is $k'$-wise independent for some $k' \geq 2$.
\end{theorem}

In the rest of this section, we set $m:=32k^2/(\eps^5\gamma^2)$.
Also, by a simple application of Chernoff bound, if we draw 
$M := O(k^{4}\log{n}\log(1/\delta)/(\eps^{9}\gamma^{4}))$
random examples from the noisy example oracle $\tilde{\mathrm{EX}}(c,D)$, 
then with probability at least $1-\delta$, we can estimate quantities such as
$\E_{\tilde{D}_1}[x_i]$, $\E_{\tilde{D}_1}[x_i\cdot x_j]$ with additive accuracy $O(1/(\eps m))$
for every $1 \leq i, j \leq n$.
To ease exposition, from now on, we condition our arguments on this event happening.
 
Since every $k'$-wise independent distribution for $k'\geq 2$ is also pairwise independent, 
it is enough to prove the theorem for $k'=2$.

Our list-learning algorithm is described in Algorithm~\ref{alg:learning},
in which call Algorithm~\ref{alg:pairwise} as a subroutine to filter out pairwise independent variables under distribution $\tilde{D}_1$.

\RestyleAlgo{boxruled}
\LinesNumbered
\begin{algorithm}[ht] 
\caption{$\textrm{Learning-Conjunction-under-Attribute-Noise}~(\tilde{\mathrm{EX}}, k, \eps, \delta)$} \label{alg:learning}
	\SetKwData{Left}{left}\SetKwData{This}{this}\SetKwData{Up}{up}
	\SetKwFunction{Decode}{Decode}
	\SetKwInOut{Input}{input}\SetKwInOut{Output}{output}
	\Input{Noisy example oracle $\tilde{\mathrm{EX}}(c,D)$, integer $k$, error parameter $\eps$, and confidence parameter $\delta$}
	\Output{A list of conjunctions}
	\BlankLine
	$m:=32k^2/(\eps^5\gamma^2)$\\
	$M := O(k^{4}\log{n}\log(1/\delta)/(\eps^{9}\gamma^{4}))$ \\
	$\mathcal{M} \leftarrow$  $M$ random labeled examples drawn from the noisy example oracle $\tilde{\mathrm{EX}}(c,D)$ \\
	$S \leftarrow \textrm{Pairwise-Independence-Test}~(\mathcal{M}, \eps, \delta)$ \\
	\For{$i \leftarrow 1$ \KwTo $n$}{
		Use $\mathcal{M}$  to estimate label sensitivity at the $i^{\text{th}}$ bit $\widehat{\mathrm{LS}}_i$ \\
		\If{$\widehat{\mathrm{LS}}_i < \eps\gamma/k$}
			{remove $i$ from $S$}\label{line:eliminate}
	}
	\eIf {$|S|< m$ \label{line:size_S}} 
		{Output the list of conjunctions $\mathbf{0} \cup \{\land_{i\in c'} x_i\}_{c' \in \binom{S}{\leq k}}$}
		{Output $\mathbf{0}$ }
\end{algorithm}

\RestyleAlgo{boxruled}
\LinesNumbered
\begin{algorithm}[ht] 
\caption{$\textrm{Pairwise-Independence-Test}~(\mathcal{M}, \eps, \delta)$} \label{alg:pairwise}
	\SetKwData{Left}{left}\SetKwData{This}{this}\SetKwData{Up}{up}
	\SetKw{Continue}{continue}
	\SetKwInOut{Input}{input}\SetKwInOut{Output}{output}
	\Input{$M$ random labeled examples $\mathcal{M}$, error parameter $\eps$, and confidence parameter $\delta$}
	\Output{A subset $S\subset [n]$ of nearly pairwise independent bits under $\tilde{D}_1$}
	\BlankLine
	$S \leftarrow [n]$ \\
	\For{$i \leftarrow 1$ \KwTo $n$}{\label{forins'}
		Use positive examples in $\mathcal{M}$ to empirically estimate $\widehat{\E_{\tilde{D}_1}[x_i]}$
	}
	\For{$i \leftarrow 1$ \KwTo $n-1$}{\label{forin1}
		\For{$j \leftarrow i+1$ \KwTo $n$}{\label{forin2}
			\If{$i \notin S$ or $j \notin S$}
				{\Continue}
			\If{$\widehat{\E_{\tilde{D}_1}[x_i]} \leq 1/(8\eps m)$ or $\widehat{\E_{\tilde{D}_1}[x_j]}\leq 1/(8\eps m)$ }
				{\Continue}
			Use sampled examples to empirically estimate $\widehat{\E_{\tilde{D}_1}[x_i\cdot x_j]}$ \\
			\If{$\left| \widehat{\E_{\tilde{D}_1}[x_i]}\cdot \widehat{\E_{\tilde{D}_1}[x_j]} - 
			    \widehat{\E_{\tilde{D}_1}[x_i\cdot x_j]}\right|> 1/(8\eps m)$ }
				{Remove both $i$ and $j$ from $S$ \label{line:remove}}
		}
	}
	{Output $S$}
\end{algorithm}

\subsection{Proof of the theorem}
In the rest of this subsection, we use the notation $\widehat H$ to denote the estimate of a quantity $H$ 
using random examples sampled from the noisy example oracle $\tilde{\mathrm{EX}}(c,D)$.

First of all, since we include the trivial functions $\mathbf{0}$ and $\mathbf{1}$ in the output list,
our learning algorithm succeed trivially whenever the target concept is $\eps$-close to either $\mathbf{0}$ or $\mathbf{1}$.
Therefore, from now on, we assume that $\eps \leq \Pr_{D}[c(x)=1] \leq 1-\eps$.

\subsubsection{Conjunction bits with low label-sensitivity}
The next lemma shows that using bits in $S$ we can get a conjunction which approximates the target concept well.
\begin{lemma}\label{lem:c_prime}
Let $c=\land_{i\in c}x_i$ be the target concept, and let $c'$ be the set of bits obtained by removing
from $c$ the set of bits eliminated in Line~\ref{line:eliminate} of Algorithm~\ref{alg:learning}.
Then conjunction $c'$ is $\eps/2$-close to $c$, i.e. $\Pr_{D}[c(x)\neq c'(x)]\leq \eps$.
\end{lemma}
\begin{proof}
First note that eliminating non-conjunction bits can not worsen the performance of our learning algorithm, so we can focus on
the effect of eliminating a conjunction bit from $S$ in Line~\ref{line:eliminate}.

Since $c'$ is a subset of $c$,
\begin{align}
\Pr_{D}[c(x)\neq c'(x)] 
& = \Pr_{D}[\text{$c'(x)=1$ and $\exists i\in c\setminus c'$ such that $x_i=0$}] \nonumber\\
& \leq \Pr_{D}[\text{$\exists i\in c\setminus c'$ such that $x_i=0$}] \nonumber\\
& \leq \sum_{i\in c\setminus c'} \Pr_{D}[x_i=0]. \qquad \text{(by union bound)} \label{eqn:union_bound}
\end{align}

We can upper bound $\Pr_{D}[x_i=0]$ for any $i\in c\setminus c'$ as
\begin{align*}
\Pr_{D}[x_i=0] &= \Pr_{D}[c(x)=0]\cdot \Pr_{D_0}[x_i=0]+\Pr_{D}[c(x)=1]\cdot \Pr_{D_1}[x_i=0]\\
               &= \Pr_{D}[c(x)=0]\cdot \Pr_{D_0}[x_i=0] \leq \Pr_{D_0}[x_i=0]. 
\end{align*}

On the other hand, in terms of quantities over the observed distribution $\tilde{D}$, we have
\begin{align*}
\Pr_{\tilde{D}_0}[x_i=0] 
&= (1-\nu_i)\Pr_{D_0}[x_i=0] + \nu_{i}\Pr_{D_0}[x_i=1] = (1-\nu_i)\Pr_{D_0}[x_i=0] + \nu_{i}(1-\Pr_{D_0}[x_i=0]) \\
&= (1-2\nu_i)\Pr_{D_0}[x_i=0]+\nu_i,
\end{align*}
and
\[
\Pr_{\tilde{D}_1}[x_i=0]=(1-\nu_i)\Pr_{D_1}[x_i=0] + \nu_{i}\Pr_{D_1}[x_i=1] 
= \nu_{i}\Pr_{D_1}[x_i=1] \leq \nu_i.
\]
Using $O(\log{n}\log(1/\delta)k^2/\eps^3\gamma^2)=o(M)$ random examples, 
we can, with probability at least $1-\delta$,
obtain $\Omega(\log{n}\log(1/\delta)k^2/\eps^2\gamma^2)$ random negative examples and 
$\Omega(\log{n}\log(1/\delta)k^2/\eps^2\gamma^2)$ random positive examples, and get an estimate of $\widehat{\mathrm{LS}}_i$
with $|\widehat{\mathrm{LS}}_i - \mathrm{LS}_{i}|\leq \eps\gamma/(2k)$ for every $1\leq i \leq n$.
Since bit-$i$ was eliminated from $S$, we 
\[
\mathrm{LS}_i \leq \widehat{\mathrm{LS}}_i + \eps\gamma/(2k) < 2\eps\gamma/k.
\]
Combining this with bounds on $\Pr_{\tilde{D}_0}[x_i=0]$ and $\Pr_{\tilde{D}_1}[x_i=0]$, we have
\[
2\eps\gamma/k > \mathrm{LS}_i = \Pr_{\tilde{D}_0}[x_i=0] - \Pr_{\tilde{D}_1}[x_i=0] \geq (1-2\nu_i)\Pr_{D_0}[x_i=0]
> 2\gamma\Pr_{D_0}[x_i=0],
\]
where the last step follows from the fact that $\nu_i<\frac12-\gamma$. Therefore we have $\Pr_{D_0}[x_i=0]<\eps/k$.

Finally, plugging the above upper bound on $\Pr_{D_0}[x_i=0]$ into inequality~\eqref{eqn:union_bound} completes the proof.
\end{proof}

\subsubsection{Pairwise independent bits}
A simple but important observation is that, if the target concept conjunction is $c=\land_{i\in c}x_i$,
then in the observed distribution $\tilde{D}_1$ of positive examples, the bits in $c$ are totally independent.
This is because, when restricting to bits in $c$, $D_1$ is supported on a single vector $1^k$.
After applying the (bit-wise independent) attribute noise, $\tilde{D}_1$ is a product distribution
when restricting to bits in $c$.

As it is computationally expensive to check total independence among the conjunction bits on $\tilde{D}_1$,
and pairwise independence suffices for our concentration argument,
we check pairwise independence in Algorithm~\ref{alg:pairwise} by estimating the covariances between each pair of bits.
\begin{lemma}\label{lem:pairwise_indep}
With probability at least $1-\delta$, the followings hold: 
the output $S$ of Algorithm~\ref{alg:pairwise} includes every bit in $c$; and conversely, every pair of
bits $X_i$ and $X_j$ in $S$ are close to being pairwise independent in the sense that
$|\cov_{\tilde{D}_1}(X_i,X_j)|\leq 1/(4\eps m)$. 
\end{lemma}

\begin{claim}\label{claim:almost_constant}
Let $D':\{0,1\}^n \to \R^{\geq 0}$ be a distribution and let $X\in \{0,1\}^n$ be the random variable 
obtained from sampling according to $D'$. 
Then, for any $0 \leq \eps \leq 1/2$, if $\Pr[X_i=1]\leq \eps$ for some $1\leq i \leq n$, 
then $|\cov(X_i, X_j)|\leq \eps$
for every $i\neq j$. The same bound holds when $\Pr[X_i=0]\leq \eps$.
\end{claim}

\begin{proof}
Let $p_0=\Pr[X_i=0 \land X_j=0]$, $p_1=\Pr[X_i=0 \land X_j=1]$, $p_2=\Pr[X_i=1 \land X_j=0]$,
and $p_3=\Pr[X_i=1 \land X_j=1]$. Then $p_2+p_3=\Pr[X_i=1]\leq \eps$
and $\cov(X_i, X_j)=p_3-(p_2+p_3)(p_1+p_3)$.
Therefore, $\cov(X_i, X_j)\geq -(p_2+p_3)(p_1+p_3) \geq -(p_2+p_3)=-\eps$.
On the other hand, $\cov(X_i, X_j)\leq p_3-p_3^2 \leq \eps-\eps^2 \leq \eps$, as
$x-x^2$ is increasing for $0\leq x\leq 1/2$.

The case of $\Pr[X_i=0]\leq \eps$ follows directly from the identity $\cov(1-X_i, 1-X_j)=\cov(X_i, X_j)$.
\end{proof}

\begin{claim}\label{claim:covariance}
Let distribution $D'$ and random variable $X$ be the same as in Claim~\ref{claim:almost_constant}.
For any pair of distinct bits $i$ and $j$, let 
$\widehat{\cov(X_i, X_j)}:= \widehat{\E[X_i\cdot X_j]} - \widehat{\E[X_i]}\cdot \widehat{\E[X_j]}$
be the estimated covariance of $X_i$ and $X_j$.
Then the estimate error can be upper bounded as
\[
|\widehat{\cov(X_i, X_j)} - \cov(X_i, X_j)| \leq |\widehat{\E[X_i\cdot X_j]}-\E[X_i\cdot X_j]|
+ 2|\widehat{\E[X_i]}-\E[X_i]| + 2|\widehat{\E[X_j]}-\E[X_j]|.
\]
\end{claim}
\begin{proof}
Let $\Delta X_i = \widehat{\E[X_i]}-\E[X_i]$ and $\Delta X_j = \widehat{\E[X_j]}-\E[X_j]$.
Then we have
\begin{align*}
\left| \widehat{\E[X_i]}\cdot \widehat{\E[X_j]} - \E[X_i]\cdot \E[X_j] \right|
& = |\Delta X_i \E[X_j] + \Delta X_j \E[X_i] + \Delta X_i \Delta X_j | \\
& \leq |\Delta X_i|(\E[X_j]+|\Delta X_j|) + |\Delta X_j|(\E[X_i]+|\Delta X_i|) \\
& \leq 2|\Delta X_i|+ 2|\Delta X_j|,
\end{align*}
because both $\widehat{\E[X_i]}$ and $\E[X_i]$ are real numbers between $0$ and $1$.
Now the bound in the claim follows directly from
\begin{align*}
\left|\widehat{\cov(X_i, X_j)} - \cov(X_i, X_j)\right|
& =\left|\widehat{\E[X_i\cdot X_j]} - \widehat{\E[X_i]}\cdot \widehat{\E[X_j]} 
  -\E[X_i\cdot X_j]+ \E[X_i]\cdot \E[X_j]  \right| \\
& \leq \left| \widehat{\E[X_i]}\cdot \widehat{\E[X_j]} - \E[X_i]\cdot \E[X_j] \right| + 
      \left|\widehat{\E[X_i\cdot X_j]}-\E[X_i\cdot X_j]\right|.\qedhere
\end{align*}
\end{proof}

\begin{proof}[Proof of Lemma~\ref{lem:pairwise_indep}]
As mentioned earlier, if we draw enough examples from the noisy example oracle, we can esitmate
quantities such as $E_{\tilde{D}_1}[X_i]$ and $E_{\tilde{D}_1}[X_i \cdot X_j]$ accurately enough.
More specifically, using 
$O(\log(1/\delta)\log{n}(\eps m)^2/\eps)=\tilde{O}(k^4\log(1/\delta)/(\eps^9 \gamma^4 ))$
random samples, with probability at least $1-\delta$, we have
$|\widehat{\E_{\tilde{D}_1}[X_i]}-\E_{\tilde{D}_1}[X_i]|\leq 1/(48\eps m)$ for every $1\leq i \leq n$
and $|\widehat{\E_{\tilde{D}_1}[X_i\cdot X_j]}-\E_{\tilde{D}_1}[X_i\cdot X_j]|\leq 1/(24\eps m)$
for every pair of distinct $1\leq i, j \leq n$.
Then for every pair of conjunction bits $i, j\in c$ or a pair of conjunction bit $i\in c$ and 
and a non-conjunction bit $j\in [n]\setminus c$, 
we always have $\cov_{\tilde{D}_1}(X_i, X_j)=0$.
By Claim~\ref{claim:covariance}, $|\widehat{\cov_{\tilde{D}_1}(X_i, X_j)}|\leq 1/(8\eps m)$,
so any conjunction bit can never be removed from $S$ in line~\ref{line:remove} of Algorithm~\ref{alg:pairwise}.
On the other hand, by Claim~\ref{claim:almost_constant} and Claim~\ref{claim:covariance}
and analogous calculations,
for any pair of bits $X_i$ and $X_j$ that are in the output $S$ of Algorithm~\ref{alg:pairwise}, 
it must be the case that 
$|\cov_{\tilde{D}_1}(X_i,X_j)|\leq 1/(4\eps m)$.
\end{proof}

\subsubsection{Bounding the size of $S$}
\begin{claim}\label{claim:expectation_difference}
For every surviving bit $X_i$ in $S$, we have
$\E_{\tilde{D}_1}[X_i]-\E_{\tilde{D}}[X_i] > \eps^2\gamma/(2k)$.  
\end{claim}
\begin{proof}
If $x_i$ is in $S$, then by a similar argument as in the proof of Lemma~\ref{lem:c_prime}, 
$\mathrm{LS}_{i} \geq \widehat{\mathrm{LS}}_i -\eps\gamma/(2k) \geq \eps\gamma/(2k)$.
Now, by the definitions of $\E_{\tilde{D}}[X_i]$ and $\E_{\tilde{D}_1}[X_i]$,
\begin{align*}
\E_{\tilde{D}_1}[X_i]-\E_{\tilde{D}}[X_i]
& = \E_{\tilde{D}_1}[X_i]-(\Pr_{\tilde{D}}[c=0]\cdot \E_{\tilde{D}_0}[X_i]+ \Pr_{\tilde{D}}[c=1]\cdot \E_{\tilde{D}_1}[X_i]) \\
& = (1-\Pr_{\tilde{D}}[c=1])(\E_{\tilde{D}_1}[X_i]- \E_{\tilde{D}_0}[X_i]) \\
& \geq (1-\Pr_{\tilde{D}}[c=1])\frac{\eps\gamma}{4k} \\
& > \eps \cdot \frac{\eps\gamma}{4k} \qquad \text{(since $\Pr_{\tilde{D}}[c=1]=\Pr_{D}[c=1]\leq 1-\eps$)} \\
& =\frac{\eps^2\gamma}{2k}.
\end{align*}
\end{proof}

\begin{lemma}\label{lem:large_S}
Suppose the size of $S$ at line~\ref{line:size_S} in Algorithm~\ref{alg:learning} is at least $m$.
Then the target concept $c$ is $\eps$-close to the all-zero function $\mathbf{0}$.
\end{lemma}
\begin{proof}
Suppose $|S|\geq m$. Let $S'\subseteq S$ be any subset of $S$ of size exactly $m$.
Without loss of generality, assume that $S'=\{1,\ldots, m\}$.

Let $X$ and $X^{+}$ be the random variables obtained by sampling from $\{0,1\}^n$
according to distributions $\tilde{D}$ and $\tilde{D}_{1}$ respectively. 
Let random variable $Z(X):=X_1+\cdots+X_m$ and $Z^{+}(X^{+}):=X^{+}_1+\cdots+X^{+}_m$.

Since $D$ is pairwise independent, then by Claim~\ref{claim:noise_indep},
distribution $\tilde{D}$ is pairwise independent as well. Therefore,
\[
\var(Z)=\var(X_1)+\cdots+\var(X_m)=\sum_{i=1}^{m}\E_{\tilde{D}}[X_i](1-\E_{\tilde{D}}[X_i])\leq \frac{m}{4}.
\]

On the other hand, using the bound on covariances in Lemma~\ref{lem:pairwise_indep}, we have
\[
\var(Z^{+})=\sum_{i=1}^{m}\var(X^{+}_i)+\sum_{i\neq j}\cov(X^{+}_i, X^{+}_j)
< \frac{m}{4}+m^2\frac{1}{4\eps m} \leq \frac{m}{2\eps}.
\]

Let $\bar{Z}=\E_{\tilde{D}}[Z]$ and $\bar{Z}^{+}=\E_{\tilde{D}_1}[Z^{+}]$. Then by Claim~\ref{claim:expectation_difference},
\[
\Delta Z:=\bar{Z}^{+}-\bar{Z}>\frac{\eps^2\gamma m}{2k}.
\]

Now, by setting $\Delta_1=\sqrt{\frac{m}{2\eps}}$ and applying Chebyshev's inequality to $Z$, we have
\[
\Pr_{\tilde{D}}[Z\geq \bar{Z}+\Delta_1]\leq \Pr[|Z-\bar{Z}|\geq \Delta_1] \leq \frac{\var(Z)}{\Delta_1^2}\leq \eps/2.
\]

Similarly, letting $\Delta_2=\sqrt{\frac{2m}{\eps}}$ and applying Chebyshev's inequality to $Z^+$ yields
\[
\Pr_{\tilde{D}_1}[Z^{+} \leq \bar{Z}^{+}-\Delta_2] \leq 1/4.
\]

It is easily checked that $\Delta_1+\Delta_2<\frac{\eps^2\gamma m}{2k} < \Delta Z$. Therefore,
\begin{align*}
\eps/2 
& \geq \Pr_{\tilde{D}}[Z(X) \geq \bar{Z}+\Delta_1] \geq \Pr_{\tilde{D}}[Z(X)\geq \bar{Z}^{+}-\Delta_2] \\
& \geq \Pr_{\tilde{D}}[Z(X)\geq \bar{Z}^{+}-\Delta_2 \text{ and $X$ is a positive example}] \\
& = \Pr_{\tilde{D}_1}[Z^{+}(X^{+})\geq \bar{Z}^{+}-\Delta_2]\Pr_{\tilde{D}}[c(X)=1]\\
& \geq (1-\frac{1}{4})\Pr_{\tilde{D}}[c(X)=1],
\end{align*}
and hence
\[
\Pr_{\tilde{D}}[c(X)=1]=\Pr_{D}[c(X)=1]\leq \frac{\eps/2}{1-1/4}=\frac{2}{3}\eps \leq\eps,
\]
which completes the proof.
\end{proof}
\ignore{
Now if we focus on the first $m$ bits of the Boolean cube $\{0,1\}^n$, the total mass of $\tilde{D}$ 
(distribution of both positive and negative examples) on strings of weight 
at least $\bar{Z}+\Delta_1$ is at most $\eps/2$. 
On the other hand, most of the mass (i.e. at least $3/4$) of $\tilde{D}_1$ (distribution of positive examples) 
is on strings of weight at least $\bar{Z}^{+}-\Delta_2>\bar{Z}+\Delta_1$.
Therefore the fraction of positive examples is at most $\frac{\eps/2}{1-1/4}=\frac{2}{3}\eps \leq\eps$; or in other words,
$\Pr_{\tilde{D}}[c(X)=1]=\Pr_{D}[c(X)=1] \leq \eps$.
} 

\subsubsection{Putting everything together}
Now we are ready to put everything together and prove the correctness of list-learning algorithm, 
i.e., Theorem~\ref{thm:learning_main}.

\begin{proof}[Proof of Theorem~\ref{thm:learning_main}]
First of all, the claimed sample complexity of the learning algorithm 
follows directly from Lemma~\ref{lem:pairwise_indep},
and the time complexity bound is due to the fact that we need to estimate, using the random examples,
$\widehat{\cov_{\tilde{D}_1}(X_i, X_j)}$ for every pair  $1\leq i < j \leq n$, 
and that at the end we may need to output a list of $\binom{m}{\leq k}$ conjunctions. 

Next, by Lemma~\ref{lem:pairwise_indep}, every conjunction bit passes the Pairwise-Independence-Test
and hence in $S$. Then, by Lemma~\ref{lem:c_prime}, filtering out low label-sensitive bits can cause
at most an error of $\eps$. That is, if we output all
$\binom{m}{\leq k}$ conjunctions of size at most $k$ from bits in $S$, at least one of these
is $\eps$-close to the target concept $c(x)$.

Finally, Lemma~\ref{lem:large_S} ensures that when the size of $S$ is large, we can simply output
the $\mathbf{0}$ function which is $\eps$-close to $c$.
\end{proof}

\section*{Acknowledgements}
EG was supported by NSF CCF-1910659 and NSF CCF-1910411.
BJ was supported by NSF award CCF-1718380.
NX was supported in part by ARO W911NF1910362.

\bibliographystyle{alpha}
\bibliography{references}
\appendix
\section*{Appendix}
\section{The trivial ``best agreement'' algorithm (information theoretic bound version)}
\label{sec:trivialPAC}

A naive algorithm for learning $k$-conjunctions with attribute noise is to try all $\sum_{i=0}^{k}2^{i}\binom{n}{i} < (2n)^{k+1}$ conjunctions of size at most $k$ and output the one that agrees with examples best.

\begin{theorem}
Given $0<\eps<1/2$ and assume the noise rate per coordinate is unknown and satisfies $\nu\leq \frac{\eps}{2k}$, the naive algorithm that outputs the $k$-conjunction with maximum agreement with the observed distribution runs in time $O(n^k)$ and with probability $1-\delta$ outputs a conjunction  that is $(1-\eps)$-close to the conjunction labeling the noisy examples.

\end{theorem}

\begin{proof}
Let $D$ be the underlying distribution and let $\boldsymbol\nu =(\nu_1, \ldots, \nu_n)$ be the attribute noise vector with upper bound $\nu$, i.e. $\nu_i \leq \nu$ for every $1\leq i \leq n$. For ease of exposition, assume that $f(x)=x_1 \land \cdots \land x_k$ is the target concept. For every $x\in \{0,1\}^n$, let $\tilde{x}=x\oplus \mu$ be the vector obtained from $x$ by adding the attribute noise $\mu$ specified by $\boldsymbol\nu$. Lastly, let $\hat{X}$ denote the set of noisy examples output by the oracle $\{\tilde{x_1}, \tilde{x_2}, \ldots, \tilde{x_m}\}$.
Define the {\em empirical disagreement} of a conjunction $g$ on the sample by 
$$\mathrm{disagreement(g)}_{\hat{X}}=\frac1m \sum_{\tilde x\in {\hat X}} I_{g(\tilde{x})\neq f(x)}, $$
where $I_{g(\tilde{x})\neq f(x)}$ is the indicator random variable of the event that $g(\tilde{x})\neq f(x)$.

By a Hoeffding bound, it follows that 
\[
\Pr[|\mathrm{disagreement(g)}_{\hat{X}}-\E_{x, \nu}[\mathrm{disagreement(g)}_{\hat{X}}]|>t ]\leq e^{-2mt^2}.
\]

Let us calculate $\E_{x, \nu}[\mathrm{disagreement(g)}_{\hat{X}}]$ first when $g=f$, and then when $dist_D(f,g)>\eps$. We will upper bound this quantity when $f=g$ and lower bound it when $f$ and $g$ are $\eps$-far. We will show that the minimum disagreement among all $\eps$-far functions $g$ is larger than the disagreement of $f$ on the observed set $\hat X$, with high probability. Therefore we output an $\eps-$close conjunction with high probability $1-\delta.$

Note that the example oracle generates an example in the following process: first draws a string $x$ according to $D$, labels it as $f(x)$, then adds the attribute noise which transforms $x$ into $\tilde{x}$. Therefore the example we see is $(\tilde{x}, f(x))$. But $f$ will predict the label as $f(\tilde{x})$. Hence, the probability that $f$ makes a mistake, i.e., the disagreement between $f$ and the example oracle is
\begin{equation}\label{eq:disagreement_UB}
\E_{x, \bnu}[\mathrm{disagreement(g)}_{\hat{X}}]=\Pr_{D,\boldsymbol\nu}[f(x)\neq f(\tilde{x})] \leq \max_{x}\Pr_{\boldsymbol\nu}[f(x)\neq f(\tilde{x})].
\end{equation}

Write $x|_{[k}]$ for the $k$-bit string obtained by projecting $x$ onto index subset $[k]$. Clearly $f(x)=1$ if and only if $x|_{[k]}=1^k$. If $f(x)=0$, then $\Pr_{\boldsymbol\nu}[f(x)\neq f(\tilde{x})]=
\Pr_{\boldsymbol\nu}[\tilde{x}|_{[k]}=1^k]=\prod_{i\in [k]: x_i=1}(1-\nu_i) \cdot \prod_{i\in [k]: x_i=0}\nu_i
\leq \prod_{i\in [k]}\nu_i \leq 1-\prod_{i\in [k]}(1-\nu_i)$,
assuming $\nu<1/2$.

On the other hand, when $f(x)=1$, then
\[
\Pr_{\boldsymbol\nu}[]f(x)\neq f(\tilde{x})]=\Pr_{\boldsymbol\nu}[\tilde{x}|_{[k]} \neq 1^k]
=1-\prod_{i\in [k]}(1-\nu_i) \leq 1-(1-\nu)^k \leq k\nu.
\]

Therefore, $\E_{x, \bnu}[\mathrm{disagreement(g)}_{\hat{X}}]\leq k\nu.$

Note that $\Pr_{\boldsymbol\nu}[g(x)\neq g(\tilde{x})] \leq k\nu$ holds for any conjunction $g$ of size at most $k$. Now for any $k$-conjunction $g$ which is at distance $\epsilon$ from $f$ under $D$, i.e. $\mathrm{dist}_{D}(f,g) =\eps$, we have

\begin{align*}
\E_{x, \bnu}[\mathrm{disagreement(g)}_{\hat{X}}]
&=\sum_{x}D(x)\Pr_{\boldsymbol\nu}[f(x)\neq g(\tilde{x})] \\   
&=\sum_{x: f(x)=g(x)}D(x)\Pr_{\boldsymbol\nu}[g(x)\neq g(\tilde{x})] + 
   \sum_{x: f(x)\neq g(x)}D(x)\Pr_{\boldsymbol\nu}[g(x) = g(\tilde{x})] \\
&\geq \sum_{x: f(x)\neq g(x)}D(x)\Pr_{\boldsymbol\nu}[g(x) = g(\tilde{x})] 
\geq (1-k\nu)\mathrm{dist}_{D}(f,g) = (1-k\nu)\eps.
\end{align*}

By taking a union bound over all the $O(n^k)$ conjunctions that are $\eps$-far from $g$, it follows that 
with probability $> 1- n^k e^{-2mt^2}$ all these conjunctions $g$ are such that 
\[
\mathrm{disagreement(g)}_{\hat{X}}\geq (1-k\nu)\eps-t.
\]
By the above calculations it also follows that $f$ itself satisfies 
\[
\mathrm{disagreement(f)}_{\hat{X}}\leq k\nu+t.
\]

It follows that if we assume that the maximum attribute noise is small enough, e.g.  $\nu\leq \frac{\eps}{2k}$, 
$t=\eps/8$, $\eps<1/2$ and $n^k e^{-2mt^2}<\delta/2$, then with probability $1-\delta$ we output a conjunction that is $\eps$-close to $f$, using $m=\Theta(\frac{1}{\eps^2}(\log \frac{1}{\delta}+k\log n))$ examples.

\end{proof}
\section{Proof of Claim \ref{claim:noise_indep}}\label{proofclaim:kwise}
\begin{proof}
First of all, for any $1\leq i \leq n$, if we let $p_i:=\Pr_{D}[X_{i}=1]$ and $\tilde{p}_i:=\Pr_{\tilde{D}}[X_{i}=1]$,
then
\[
\begin{pmatrix}
1-\tilde{p}_i\\
\tilde{p}_i
\end{pmatrix}
=\begin{pmatrix}
1-\nu_i & \nu_i\\
\nu_i & 1-\nu_i
\end{pmatrix} 
\begin{pmatrix}
1-p_i\\
p_i
\end{pmatrix}.
\]
More generally, for any subset of $k$ indices $\{i_{1}, \ldots, i_{k}\} \subset [n]$,
\[
\begin{pmatrix}
\Pr_{\tilde{D}}[X_{i_{1}}\cdots X_{i_{k}}=0^k]\\
\vdots \\
\Pr_{\tilde{D}}[X_{i_{1}}\cdots X_{i_{k}}=1^k]
\end{pmatrix}
= \begin{pmatrix}
1-\nu_{i_1} & \nu_{i_1}\\
\nu_{i_1} & 1-\nu_{i_1}
\end{pmatrix} \otimes \cdots \otimes
\begin{pmatrix}
1-\nu_{i_k} & \nu_{i_k}\\
\nu_{i_k} & 1-\nu_{i_k}
\end{pmatrix}
\begin{pmatrix}
\Pr_{D}[X_{i_{1}}\cdots X_{i_{k}}=0^k]\\
\vdots \\
\Pr_{D}[X_{i_{1}}\cdots X_{i_{k}}=1^k]
\end{pmatrix},
\]
where $\otimes$ stands for the Kronecker product of matrices.
Now suppose that $D$ is $k$-wise independent, then 
\[
\begin{pmatrix}
\Pr_{D}[X_{i_{1}}\cdots X_{i_{k}}=0^k]\\
\vdots \\
\Pr_{D}[X_{i_{1}}\cdots X_{i_{k}}=1^k]
\end{pmatrix}
=\begin{pmatrix}
1-p_{i_1}\\
p_{i_1}
\end{pmatrix} \otimes \cdots \otimes
\begin{pmatrix}
1-p_{i_k}\\
p_{i_k}
\end{pmatrix},
\]
and it follows that
\begin{align*}
\begin{pmatrix}
\Pr_{\tilde{D}}[X_{i_{1}}\cdots X_{i_{k}}=0^k]\\
\vdots \\
\Pr_{\tilde{D}}[X_{i_{1}}\cdots X_{i_{k}}=1^k]
\end{pmatrix}
& = \left(\begin{pmatrix}
1-\nu_{i_1} & \nu_{i_1}\\
\nu_{i_1} & 1-\nu_{i_1}
\end{pmatrix} 
\begin{pmatrix}
1-p_{i_1}\\
p_{i_1}
\end{pmatrix}
\right)\otimes \cdots \otimes
\left(\begin{pmatrix}
1-\nu_{i_k} & \nu_{i_k}\\
\nu_{i_k} & 1-\nu_{i_k}
\end{pmatrix} 
\begin{pmatrix}
1-p_{i_k}\\
p_{i_k}
\end{pmatrix}
\right)\\
& = 
\begin{pmatrix}
1-\tilde{p}_{i_1}\\
\tilde{p}_{i_1}
\end{pmatrix} \otimes \cdots \otimes
\begin{pmatrix}
1-\tilde{p}_{i_k}\\
\tilde{p}_{i_k}
\end{pmatrix}.
\end{align*}
That is, $\tilde{D}$ is also $k$-wise independent.
The other direction follow from an identical argument by noting that matrix
$\begin{pmatrix}
1-\nu_i & \nu_i\\
\nu_i & 1-\nu_i
\end{pmatrix}$ is invertible --- namely
\[
\begin{pmatrix}
1-\nu_i & \nu_i\\
\nu_i & 1-\nu_i
\end{pmatrix}^{-1} =
\begin{pmatrix}
\frac{1-\nu_i}{1-2\nu_i} & -\frac{\nu_i}{1-2\nu_i}\\
-\frac{\nu_i}{1-2\nu_i} & \frac{1-\nu_i}{1-2\nu_i}
\end{pmatrix},  
\]
for every $0\leq \nu_i <1/2$.
\end{proof}

\end{document}